\documentclass[%
 reprint,
 superscriptaddress,
 amsmath,amssymb,
 aps,
 prr,
]{revtex4-2}

\usepackage{graphicx}% Include figure files
\usepackage{dcolumn}% Align table columns on decimal point
\usepackage{bm}% bold math
\usepackage{hyperref}% add hypertext capabilities
\hypersetup{% https://tex.stackexchange.com/a/847/23046
    breaklinks=true,   % splits links across lines (fixes issue when built with arxiv)
    colorlinks,
    linkcolor={red!50!black},
    citecolor={blue!50!black},
    urlcolor={blue!80!black}
}
\usepackage{multirow}
\usepackage[linesnumbered, ruled, vlined]{algorithm2e}
\usepackage{graphicx}% Include figure files
\usepackage{subcaption} % throws warning
\usepackage{blkarray}
\usepackage{tabularx}
\usepackage{listings}
\usepackage{amsthm}
\usepackage{cleveref} % Has known bug: https://github.com/AASJournals/AASTeX60/issues/69#issuecomment-389945092
\usepackage{cancel}
\usepackage{booktabs}
\usepackage{amsfonts}
\usepackage{bbm}
\usepackage{tikz}
\usetikzlibrary{shapes}
\usetikzlibrary{shapes.geometric}
\usetikzlibrary{positioning}
\usetikzlibrary{calc, arrows.meta}
\usepackage{pgfplots}
\pgfplotsset{compat=1.18}
\usepackage[inline]{enumitem} % https://tex.stackexchange.com/a/146311/23046
\usepackage{mathtools}

\DeclareMathOperator*{\argmax}{arg\,max}

\theoremstyle{definition}

\newtheorem{definition}{Definition}[section]
\newtheorem{theorem}{Theorem}[section]
\newtheorem{corollary}{Corollary}[section]

\newcommand{\contract}{{\texttt{con}}}

\newcommand{\ve}{E = e}  % evidence
\newcommand{\vq}{Q = q}  % query
\newcommand{\vs}{\boldsymbol{s}}  % any
\newcommand{\vm}{M = m}  % marginalized
\newcommand{\dom}[1]{\mathcal{D}_{#1}}

\newcommand{\slice}[3]{{#1_{#2 | #3}}}
\newcommand{\tropicalcontract}{\texttt{tcon}}

\let\emptyset\varnothing

\definecolor{c01}{HTML}{4477AA}
\definecolor{c02}{HTML}{EE6677}
\definecolor{c03}{HTML}{228833}
\definecolor{c04}{HTML}{CCBB44}
\definecolor{c05}{HTML}{66CCEE}
\definecolor{c06}{HTML}{AA3377}
\definecolor{c07}{HTML}{BBBBBB}
\definecolor{c08}{HTML}{BBBBBB}

\begin{document}

\preprint{APS/123-QED}

\title{Probabilistic Inference in the Era of Tensor
Networks\texorpdfstring{\\}{} and Differential Programming}

\author{Martin Roa-Villescas}
 \email{m.roa.villescas@tue.nl}
\affiliation{Eindhoven University of Technology, Eindhoven, The~Netherlands}%
\author{Xuanzhao Gao}
 %\email{xz.gao@connect.ust.hk}
\affiliation{Hong Kong University of Science and Technology (Guangzhou), Guangzhou, China}%
\author{Sander Stuijk}%
 %\email{s.stuijk@tue.nl}
\author{Henk Corporaal}%
 %\email{h.corporaal@tue.nl}
\affiliation{Eindhoven University of Technology, Eindhoven, The~Netherlands}%
\author{Jin-Guo Liu} % move to second place
 %\email{jinguoliu@hkust-gz.edu.cn}
\affiliation{Hong Kong University of Science and Technology (Guangzhou), Guangzhou, China}%

\date{\today}% It is always \today, today,

\begin{abstract}

Probabilistic inference is a fundamental task in modern machine learning.
Recent advances in tensor network (TN) contraction algorithms have enabled the
development of better exact inference methods. However, many common inference
tasks in probabilistic graphical models (PGMs) still lack corresponding
TN-based adaptations. In this work, we advance the connection between PGMs and
TNs by formulating and implementing tensor-based solutions for the following
inference tasks: (i)~computing the partition function, (ii)~computing the
marginal probability of sets of variables in the model, (iii)~determining the
most likely assignment to a set of variables, and (iv)~the same as (iii) but
after having marginalized a different set of variables. We also present a
generalized method for generating samples from a learned probability
distribution. Our work is motivated by recent technical advances in the fields
of quantum circuit simulation, quantum many-body physics, and statistical
physics. Through an experimental evaluation, we demonstrate that the
integration of these quantum technologies with a series of algorithms
introduced in this study significantly improves the effectiveness of existing
methods for solving probabilistic inference tasks.

\end{abstract}

\maketitle

%\tableofcontents

Probabilistic inference is a fundamental component of machine learning. It
enables machines to reason, predict, and assist experts in making decisions
under uncertain conditions. The main challenge in applying exact inference
techniques lies in the explosion of the computational cost as the number
of variables involved increases. Unfortunately, modeling real-world problems
often demands a high number of variables. Because of this, performing
probabilistic inference remains an intractable endeavor in many practical
applications.

In the past decades, several methods have been developed to enhance the
computational efficiency of exact inference in complex models. Clustering
methods, which include the family of junction tree
algorithms~\cite{lauritzen1988local,jensen1990bayesian}, Symbolic
probabilistic
inference~\cite{shachter1990symbolic,li1994efficient,gehr2016psi}, weighted
model counting~\cite{chavira2008probabilistic,holtzen2020scaling}, and
differential-based methods~\cite{darwiche2003differential,darwiche2020advance}
stand out as popular approaches.
% Differential-based methods map probabilistic inference, especially parameter
% learning, into tensor computations to leverage the significant advancements
% in tensor-based technologies in recent years. The concept of using tensors
% to model probability distributions in vast spaces has been extensively
% studied in quantum physics~\cite{orus2014practical}.

Tensor networks (TNs), widely used in quantum many-body physics and quantum
computation~\cite{nielsen2010quantum}, are gaining increasing attention in the
machine learning community. These networks have been shown to be an
exceptionally powerful framework for modeling many-body quantum
states~\cite{orus2014advances}. Notable examples of TNs include Matrix Product
States (MPS)~\cite{perez2007matrix}, Tree Tensor Networks
(TTN)~\cite{shi2006classical}, Multi-scale Entanglement Renormalization Ansatz
(MERA)~\cite{vidal2007entanglement}, and Projected Entangled Pair States
(PEPS)~\cite{verstraete2004renormalization}. In recent years, they have become
increasingly popular for classical benchmarking of quantum computing
devices~\cite{arute2019quantum,markov2008simulating,pan2022simulation,gao2021limitations}.
The application of TNs in machine learning has been primarily focused on
generative modeling, aiming to learn a model's joint probability distribution
from data and generate samples from it. For instance, Han et
al.~\cite{han2018unsupervised} proposed using an MPS network for this purpose,
ensuring that the TN topology is constrained to a chain-like structure.
Building on this idea, Cheng~et~al.~\cite{cheng2019tree} advocated for the use
of a TTN over an MPS network, aiming to enhance representational capabilities
and to improve efficiency in both training and sampling.

While there have been notable advancements in understanding the theoretical
duality between TNs and PGMs~\cite{robeva2019duality} and the integration of
several TN techniques for generative
sampling~\cite{han2018unsupervised,cheng2019tree}, many common probabilistic
tasks in PGMs still lack corresponding TN-based adaptations. In this work, we
bridge the gap between PGMs and TNs further by formulating and implementing
tensor-based solutions for a series of important probabilistic tasks.
Specifically, given evidence for a subset of the variables in the model, we
formulate and provide TN-based implementations for computing:
\begin{enumerate*}[label=(\roman*)]
  \item the partition function,
  \item the marginal probability of sets of variables, 
  \item the most likely assignment to a set of variables,
  \item the most likely assignment to a set of variables after marginalizing a
    different set, and
  \item unbiased variable sampling, which generalizes the work of
    Han~et~al.~\cite{han2018unsupervised} and Cheng
    et~al.~\cite{cheng2019tree}.
\end{enumerate*}
Our work introduces a novel unity-tensor approach to compute marginal
probabilities and the most likely assignment. This technique significantly
reduces the computational cost of calculating multiple marginal probabilities.

Inspired by recent technical progress in the fields of quantum circuit
simulation, quantum many-body physics, and statistical physics, our research
aims to capitalize on these advancements. We employ new hyper-optimized
contraction order finding algorithms~\cite{gray2021hyper,kalachev2022multi}
that have evolved in classical benchmarking quantum computing
devices~\cite{arute2019quantum,markov2008simulating,pan2022simulation,gao2021limitations}.
These hyper-optimized contraction ordering algorithms optimize both the
computation time and runtime memory usage, resulting in a significant
improvement in performance. This work also benefits from the latest advances
of tropical tensor networks~\cite{liu2021tropical} followed by the
introduction of generic tensor networks~\cite{liu2022computing}, which allow
us to seamlessly devise performant solutions for the different inference tasks
described earlier by adjusting the element types of a consistent tensor
network. Our implementation leverages cutting-edge developments commonly found
in tensor network libraries, including a highly optimized set of BLAS
routines~\cite{blackford2002updated,liu2023tropical} and GPU technology. 

We present experimental results demonstrating that our tensor-based
implementation is highly effective in advancing current methods for solving
probabilistic inference tasks. Our library exhibits speedups of three to four
orders of magnitude compared to a series of established solvers for the
following exact inference tasks: computing the partition function (PR), the
marginal probability distribution over all variables given evidence (MAR), the
most likely assignment to all variables given evidence (MPE), and the most
likely assignment to the query variables after marginalizing out the remain
variables (MMAP). Furthermore, we present experimental results indicating that
by employing a GPU instead of a CPU, our proposed implementation can
accelerate the inference of MMAP tasks by up to two orders of magnitude when
the problem's computational cost exceeds a certain threshold. The ability of
our library to facilitate the seamless use of a GPU instead of a CPU for
solving probabilistic inference tasks represents a significant advantage. The
source code for the methods described in this paper is available in a Julia
package by the name of \texttt{TensorInference.jl}~\cite{roa2023tensor},
licensed under the MIT open-source license.
%The package can be accessed at
%\url{https://github.com/TensorBFS/TensorInference.jl}. 

The remainder of this paper is organized as follows. \Cref{sec:tn} provides a
review of tensor networks, laying the foundational concepts necessary for
understanding subsequent discussions. In
\Cref{sec:tensor-network-for-probabilistic-modeling}, we delve into the
formulation of various probabilistic modeling tasks in terms of tensor network
contractions, including the Partition Function (\Cref{sec:pr}), the Marginal
Probability (\Cref{sec:mar}), the Most Probable Explanation (\Cref{sec:mpe}),
the Maximum Marginal a Posteriori (\Cref{sec:mmap}), and Sampling
(\Cref{sec:sampling}). \Cref{sec:benchmarks} presents benchmarks and empirical
results to demonstrate the practical implications of our approach. Finally, we
conclude the paper in \Cref{sec:conclusions}, where we discuss
the implications, limitations, and potential future directions of our work.

\section{Tensor networks} \label{sec:tn}

Tensor networks serve as a fundamental tool for modeling and analyzing
correlated systems. This section reviews their fundamental concepts.
%, highlighting their connection with probabilistic graphical models.

A tensor is a mathematical object that generalizes scalars, vectors, and
matrices. It can have multiple dimensions and is used to represent data in
various mathematical and physical contexts. It is formally defined as follows:

\begin{definition}[Tensor]
  A tensor $T$ associated to a set of discrete variables $V$ is defined as a
  function that maps each possible instantiation of the variables in its scope
  $\dom{V} = \prod_{v\in V} \dom{v}$ to an element in the set $\mathcal{E}$,
  where $\dom{v}$ is the set of all possible values that the variable $v$ can
  take. The function $T_V$ is given by
  \begin{equation}
    T_{V}: \prod_{v \in V} \dom{v} \rightarrow \mathcal{E}.
  \end{equation}
  Within the context of probabilistic modeling, the elements in $\mathcal{E}$
  are non-negative real numbers, while in other scenarios, they can be of
  generic types.
\end{definition}

Tensors are typically represented as multidimensional arrays, where each
dimension is assigned a specific label or name. In probabilistic modeling,
these labels correspond to \emph{random variables} (or \emph{variables} for
short), and hence these terms will be used interchangeably throughout the rest
of this paper. The collective set of variables upon which a tensor operates is
known as its \emph{scope}. Before introducing the definition of a tensor
network, it is important to define the concept of \emph{slicing} (or
\emph{indexing}) tensors based on variable assignments. Let $T_V$ be a tensor
defined over the set of variables $V$. Let $M$ be another set of variables
with an arbitrary relationship to the set $V$, i.e., $M$ and $V$ may have all,
some, or no elements in common, or one may be a subset of the other. The
notation $M = m$ denotes the assignment of specific values denoted by $m$ to
the variables in $M$. The operation of slicing a tensor, denoted as
$\slice{T}{V}{M=m}$, involves evaluating the tensor $T_V$ according to the
assignment $M = m$. This operation effectively reduces the dimensions of $T_V$
by constraining it to the subspace where $M = m$. Note that if $V$ and $M$ are
disjoint, $T_V$ remains unchanged.

We now turn our attention to the formal definition of a \emph{tensor network}. 

\begin{definition}[Tensor Network~\cite{liu2022computing, cirac2021matrix, orus2014practical}] \label{def:tnet}
 A tensor network is a mathematical framework for defining multilinear maps,
 which can be represented by a triple
  $\mathcal{N} = (\Lambda, \mathcal{T}, V_0)$, where:
  \begin{itemize}
    \item $\Lambda$ is the set of variables present in the network
      $\mathcal{N}$.
    \item $\mathcal{T} = \{ T_{V_k} \}_{k=1}^{K}$ is the set of
      input tensors, where each tensor $T_{V_k}$ is associated with the
      labels $V_k$.
    \item $V_0$ specifies the labels of the output tensor.
  \end{itemize}
  Specifically, each tensor $T_{V_k} \in \mathcal{T}$ is labeled by a set of
  variables $V_k \subseteq \Lambda$, where the cardinality $|V_k|$ equals the
  rank of $T_{V_k}$. The multilinear map, or the \textbf{contraction}, applied
  to this triple is defined as
  \begin{equation} \label{eq:contraction-definition}
    T_{V_0} = \contract(\Lambda, \mathcal{T}, V_0) \overset{\mathrm{def}}{=} \sum_{m \in \dom{\Lambda
    \setminus V_0}} \prod_{T_V \in \mathcal{T}} \slice{T}{V}{M=m},
  \end{equation}
  where $M = \Lambda \setminus V_0$.
\end{definition}

For instance, matrix multiplication can be described as the contraction of a
tensor network given by
\begin{equation}
  (AB)_{\{i, k\}} = \contract\left(\{i,j,k\}, \{A_{\{i, j\}}, B_{\{j, k\}}\}, \{i, k\}\right),
\end{equation}
where matrices $A$ and $B$ are input tensors containing the variable sets
$\{i, j\}, \{j, k\}$, respectively, which are subsets of $\Lambda = \{i, j,
k\}$. The output tensor is comprised of variables $\{i, k\}$ and the summation
runs over variables $\Lambda \setminus \{i, k\} = \{j\}$. The contraction
corresponds to
\begin{equation} \label{eq:matrix-multiplication-contraction}
  (A B)_{\{i, k\}} = \sum_j
  A_{\{i,j\}}B_{\{j, k\}}.
\end{equation}

\Cref{def:tnet} introduces a minor generalization of the standard tensor
network definition commonly used in physics. It allows a label to appear more
than twice across the tensors in the network, deviating from the conventional
practice of restricting each label to two appearances. This generalized form,
while maintaining the same level of representational power, has been
demonstrated to potentially reduce the network's
treewidth~\cite{liu2022computing}, a metric that measures its connectivity.

Diagrammatically, a tensor network can be represented as an \textit{open
hypergraph}, where each tensor is mapped to a vertex and each variable is
mapped to a hyperedge. Two vertices are connected by the same hyperedge if and
only if they share a common variable. The diagrammatic representation of the
matrix multiplication shown in \Cref{eq:matrix-multiplication-contraction} is
given as follows: 
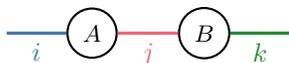
\begin{center}
  \begin{tikzpicture}[
    mytensor/.style={
      circle,
      thick,
      draw=black!100,
      font=\small,
      minimum size=0.5cm
    },
    myedge/.style={
      very thick,
    },
    ]
    \matrix[row sep=0.8cm,column sep=0.8cm,ampersand replacement= \& ] {
      \node (1) {};                                               \&
      \node (a) [mytensor] {$A$};                                 \&
      \node (b) [mytensor] {$B$};                                 \&
      \node (2) {};                                               \&
                                                                  \\
    };
    \draw [myedge, color=c01] (1) edge node[below] {$i$} (a);
    \draw [myedge, color=c02] (a) edge node[below] {$j$} (b);
    \draw [myedge, color=c03] (b) edge node[below] {$k$} (2);
  \end{tikzpicture}
\end{center}
Here, we use different colors to denote different hyperedges. Hyperedges for
$i$ and $k$ are left open to denote variables of the output tensor. A somewhat
more complex example of this is as follows:
\begin{align} \label{eq:more-complex-contraction}
  \begin{split}
    &\contract(\{i,j,k,l,m,n\}, \\
    & \quad\quad\{A_{\{i, l\}}, B_{\{l\}}, C_{\{k, j, l\}}, D_{\{k, m, n\}}, E_{\{j, n\}}\},\\
    & \quad\quad\{i,m\}) \\
    & =\sum_{j,k,l,n}A_{\{i,l\}} B_{\{l\}} C_{\{k,l\}} D_{\{k, m\}} E_{\{j, n\}}.
  \end{split}
\end{align}
Note that the variable $l$ is shared by three tensors, making regular edges,
which by definition connect two nodes, insufficient for its representation.
This motivates the need for hyperedges, which can connect a single variable to
any number of nodes. The hypergraph representation is given as:
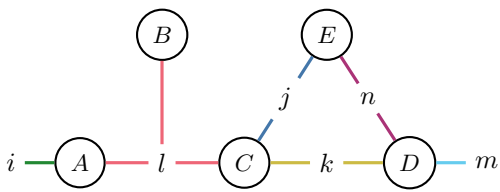
\begin{center}
  \begin{tikzpicture}[
    mytensor/.style={
      circle,
      thick,
      draw=black!100,
      font=\small,
      minimum size=0.5cm
    },
    myedge/.style={
      very thick,
    },
    ]
    \matrix[row sep=1.0cm,column sep=0.4cm,ampersand replacement= \& ] {
                                  \&
                                  \&
      \node (b) [mytensor] {$B$}; \&
                                  \&
      \node (e) [mytensor] {$E$}; \&
                                  \&
                                  \&
                                    \\
      \node (i) {$i$};            \&
      \node (a) [mytensor] {$A$}; \&
      \node (l) {$l$};            \&
      \node (c) [mytensor] {$C$}; \&
      \node (k) {$k$};            \&
      \node (d) [mytensor] {$D$}; \&
      \node (m) {$m$};            \&
                                    \\
    };
    \path (c) -- (e) node[midway] (j) {$j$};
    \path (d) -- (e) node[midway] (n) {$n$};
    \draw [myedge, color=c02] (b) edge (l);
    \draw [myedge, color=c03] (i) edge (a);
    \draw [myedge, color=c02] (a) edge (l);
    \draw [myedge, color=c02] (l) edge (c);
    \draw [myedge, color=c04] (c) edge (k);
    \draw [myedge, color=c04] (d) edge (k);
    \draw [myedge, color=c05] (d) edge (m);
    \draw [myedge, color=c01] (c) edge (j);
    \draw [myedge, color=c01] (e) edge (j);
    \draw [myedge, color=c06] (d) edge (n);
    \draw [myedge, color=c06] (e) edge (n);
  \end{tikzpicture}
\end{center}

We would now like to stress an important property of tensor networks, namely
the \emph{contraction order}. While the summations
in~\Cref{eq:more-complex-contraction} (over $j, k, l$ and $n$) can be carried
out in any order without affecting the contraction result, the order in which
these summations are performed significantly impacts the computational cost
required to contract the network. Finding the optimal order of variables to be
contracted in a tensor network is crucial for overall efficiency. To minimize
the computational cost of a TN contraction, one must optimize over the
different possible orderings of pairwise contractions and find the optimal
case~\cite{orus2014practical}. This problem is NP-hard. However, several
efficient heuristic contraction order finding algorithms~\cite{gray2021hyper,
kalachev2022multi} have been developed by the community. Given a contraction
order, each pairwise contraction can be further decomposed into a series of
BLAS operations, which are highly optimized for modern
hardware~\cite{roa2023scaling}.

To illustrate this point, consider the two contraction orders specified below
for evaluating \Cref{eq:more-complex-contraction} using binary trees:
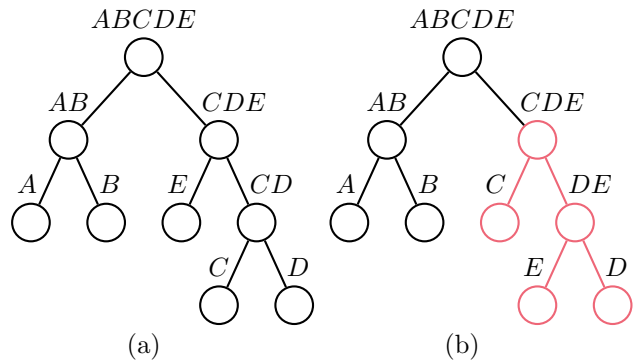
\newcommand{\vspacing}{1.10} % Adjust this value to change the vertical distance
\tikzset{
  bctensor/.style={
    circle,
    thick,
    draw=black!100,
    font=\small,
    minimum size=0.5cm
  },
  bcedge/.style={
    thick,
  }
}

\begin{center}
  \begin{tikzpicture}
    \node (abcde) [bctensor,label=above:{$ABCDE$}] at (0, 0) {};

    \node (ab) [bctensor,label=above:{$AB$}] at (-1, -\vspacing) {};
    \node (cde) [bctensor,label={[xshift=0.2cm]above:{$CDE$}}] at (1, -\vspacing) {$$};

    \node (a) [bctensor,label={[xshift=-0.05cm]above:{$A$}}] at (-1.5, -2*\vspacing) {};
    \node (b) [bctensor,label={[xshift=0.05cm]above:{$B$}}] at (-0.5, -2*\vspacing) {};
    \node (e) [bctensor,label={[xshift=-0.05cm]above:{$E$}}] at (0.5, -2*\vspacing) {};
    \node (cd) [bctensor,label={[xshift=0.2cm]above:{$CD$}}] at (1.5, -2*\vspacing) {};
                                
    \node (c) [bctensor,label=above:{$C$}] at (1.0, -3*\vspacing) {};
    \node (d) [bctensor,label={[xshift=0.05cm]above:{$D$}}] at (2.0, -3*\vspacing) {};
    \draw [bcedge] (c) edge (cd);
    \draw [bcedge] (d) edge (cd);
    \draw [bcedge] (a) edge (ab);
    \draw [bcedge] (b) edge (ab);
    \draw [bcedge] (e) edge (cde);
    \draw [bcedge] (cd) edge (cde);
    \draw [bcedge] (cde) edge (abcde);
    \draw [bcedge] (ab) edge (abcde);
    \node at (0, -3.5*\vspacing) {(a)};
  \end{tikzpicture}
  \begin{tikzpicture}
    \node (abcde) [bctensor,label=above:{$ABCDE$}] at (0, 0) {};

    \node (ab) [bctensor,label=above:{$AB$}] at (-1, -\vspacing) {};
    \node (cde) [bctensor,color=c02,label={[xshift=0.2cm]above:{$CDE$}}] at (1, -\vspacing) {$$};

    \node (a) [bctensor,label={[xshift=-0.05cm]above:{$A$}}] at (-1.5, -2*\vspacing) {};
    \node (b) [bctensor,label={[xshift=0.05cm]above:{$B$}}] at (-0.5, -2*\vspacing) {};
    \node (c) [bctensor,color=c02,label={[xshift=-0.05cm]above:{$C$}}] at (0.5, -2*\vspacing) {};
    \node (de) [bctensor,color=c02,label={[xshift=0.2cm]above:{$DE$}}] at (1.5, -2*\vspacing) {};
                                
    \node (e) [bctensor,color=c02,label={[xshift=-0.05cm]above:{$E$}}] at (1.0, -3*\vspacing) {};
    \node (d) [bctensor,color=c02,label={[xshift=0.05cm]above:{$D$}}] at (2.0, -3*\vspacing) {};
    \draw [bcedge,color=c02] (e) edge (de);
    \draw [bcedge,color=c02] (d) edge (de);
    \draw [bcedge] (a) edge (ab);
    \draw [bcedge] (b) edge (ab);
    \draw [bcedge,color=c02] (c) edge (cde);
    \draw [bcedge,color=c02] (de) edge (cde);
    \draw [bcedge] (cde) edge (abcde);
    \draw [bcedge] (ab) edge (abcde);
    \node at (0, -3.5*\vspacing) {(b)};
  \end{tikzpicture}
\end{center}

These binary trees specify the order of pairwise contractions between tensors,
starting from the bottom (leaves) and moving to the top (root). Each leaf
represents an initial tensor, and each internal node results from contracting
two child tensors and summing out any indices not needed for later operations.
Note that directly evaluating \Cref{eq:more-complex-contraction} requires
$O(n^6)$ time, where $n$ is the dimension of each variable. In contrast, both
contraction orders illustrated above reduce the time complexity to $O(n^4)$.
However, there is an important difference between these two orders,
highlighted in red. Specifically, contraction order (b) is preferred over (a)
due to its lower \emph{space complexity}, which is determined by the highest
rank among all intermediate tensors. For example, the intermediate tensor $CD$
in order (a) has a rank of 4, whereas tensor $DE$ in order (b) has a rank of
3. Lower space complexity helps to reduce the memory usage bottleneck in
tensor network computations.

\section{Tensor networks for probabilistic modeling}
\label{sec:tensor-network-for-probabilistic-modeling}

Probabilistic graphical models (PGMs) are a class of models that use graphs to
represent complex dependencies between random variables and reason about them,
with Bayesian networks, Markov random fields, and factor graphs being among
the most prevalent examples. While tensor networks and PGMs share a conceptual
foundation in representing multivariate relationships graphically, they have
traditionally evolved in parallel within distinct fields of study. Despite
their different origins, both frameworks exhibit remarkable similarities in
structure and functionality. They are both used to decompose high-dimensional
objects into a network or graph of simpler, interconnected components. The aim
of this section is to reformulate probabilistic modeling in terms of tensor
networks, thereby allowing the field of probabilistic modeling to leverage the
remarkable developments achieved in tensor network modeling in recent years.

The analyses and discussions in this section are framed within the context of
a probabilistic model. This model is characterized by a set of variables
$\Lambda$ with a corresponding joint \emph{probability mass function}
$p(\Lambda)$. Furthermore,  $p(\Lambda)$ is represented as a product of
tensors in $\mathcal{T}$, where each tensor $T_V \in \mathcal{T}$ is
associated with a subset of variables $V$ in $\Lambda$. Within $\Lambda$, we
distinguish three disjoint subsets of variables: the \emph{query} variables
$Q$, representing our variables of interest; the \emph{evidence} variables
$E$, which denote observed variables; and the \emph{nuisance} variables $M$,
which include the remaining variables.

In what follows, we present the formulation of prevalent probabilistic
inference tasks through the application of tensor network methodologies. These
tasks include:
\begin{enumerate}
  \item Calculating the \emph{partition function} (PR), also referred to as
    the \emph{probability of evidence} (\Cref{sec:pr}).
  \item Computing the marginal probability distribution over sets of variables
    given evidence (MAR) (\Cref{sec:mar}).
  \item Finding the most likely assignment to all variables given evidence,
    formally referred to as the \emph{most probable explanation} (MPE)
    (\Cref{sec:mpe}).
  \item Finding the most likely assignment to a set of query variables after
    marginalizing out the remaining variables, also known as the \emph{Maximum
    Marginal a Posteriori} (MMAP) estimate (\Cref{sec:mmap}).
  \item Generating samples from the learned distribution given evidence, also
    known as \emph{generative modeling} (\Cref{sec:sampling}).
\end{enumerate}
For more information about these tasks, refer to the website of the UAI
2022 Probabilistic Inference Competition~\cite{dechter2022uai}.

\subsection{Partition Function (PR)} \label{sec:pr}

The partition function is a central concept in statistical mechanics and
probabilistic graphical models. In statistical mechanics, it sums over all
possible states of a system, weighted by their energy, to derive key
thermodynamic quantities. In probabilistic models, it not only normalizes the
joint probability distribution, ensuring that the probabilities of all
possible outcomes sum to one, but also facilitates model comparison by
providing a measure of how well each model explains the observed data. 

Suppose we are given some evidence $e$ observed over a set of variables $E
\subseteq \Lambda$. The partition function is calculated by summing the joint
distribution $p$ over all possible values of the variables $M \subseteq
\Lambda$ that are not in $E$, i.e., $E \cap M = \emptyset$. Thus, the
partition function corresponds to:
\begin{equation}
  p(\ve) = \sum_{m \in \mathcal{D}_{\Lambda \setminus E}} p(\ve, \vm).
\end{equation}
Let us denote the set of tensors associated with the variables in $\Lambda$ as
$\mathcal{T} = \{T_V\}$, where $T_V$ is a tensor associated with the variables
$V \subseteq \Lambda$.
The partition function can be expressed as the tensor network contraction
given by:
\begin{equation}
  p(\ve) = \sum_{m \in \dom{\Lambda \setminus E}} \prod_{T_V \in \mathcal{T}} \slice{T}{V}{E=e, M=m},
\end{equation}
where $\mathcal{T}_{{\ve}} = \{\slice{T}{V}{E=e} \mid T_V \in \mathcal{T}\}$
is the set of tensors sliced over the fixed values of the evidence variables.
We can express this operation more succinctly using the definition of a
tensor contraction shown in \Cref{eq:contraction-definition}, which results in:
\begin{equation} \label{eq:pr}
  p(\ve) = \contract(\Lambda \setminus E, \mathcal{T}_{\ve}, \emptyset).
\end{equation}
Here, since evidence variables are fixed, the contraction is performed over
the remaining variables in $\Lambda \setminus E$. Correspondingly, the tensors
in $\mathcal{T}_{\ve}$ are sliced according to the evidence. To obtain the
marginal probability, all remaining variables are marginalized, leaving the
output tensor with an empty label set $\emptyset$.

\subsection{Marginal Probability (MAR)} \label{sec:mar}

The marginal probability (MAR) task involves computing the conditional
probability distribution for the set of query variables $Q$, based on known
information about the evidence variables $E$, i.e.~$p(Q \mid \ve)$. This
process requires marginalizing out nuisance variables $M$ from the joint
distribution $p(\Lambda)$, effectively averaging their impact within the joint
probability distribution. Such averaging is crucial as it accounts for the
indirect effect of these variables on the resulting marginal probabilities,
thereby enabling predictions and informed decision-making with limited
information. In what follows, we introduce a novel tensor-based algorithm for
efficiently computing marginal probability distributions over multiple sets of
query variables, which demonstrates improved performance over traditional
approaches, such as the junction tree algorithm~\cite{lauritzen1988local}.

The marginal probability query, given some evidence $\ve$, computes the
conditional distribution over the query variables $Q \subseteq \Lambda$. This
is denoted as $p(\vq \mid \ve)$, where $q \in \dom{Q}$ and it is ensured that
$E \cap Q = \emptyset$. The marginal probability can be obtained as follows:
\begin{equation}\label{eq:mar-prob}
  \begin{split}
    p(Q \mid \ve) = \frac{p(Q, \ve)}{p(\ve)}.
  \end{split}
\end{equation}
The numerator $p(Q, \ve)$ corresponds to the joint marginal probability of
configurations. This is given by the following equation:
\begin{equation}
  p(Q, \ve) = \sum_{m \in \dom{\Lambda \setminus (Q, E)}} p(Q, \ve, \vm),
\end{equation}
or, equivalently, by the following tensor network contraction:
\begin{equation}\label{eq:evidence}
    p(Q, \ve) = \contract(\Lambda \setminus E, \mathcal{T}_{\ve}, Q).
\end{equation}
The denominator $p(\ve)$ in \Cref{eq:mar-prob} corresponds to the partition
function, calculated according to \Cref{eq:pr}.

Consider the scenario where we want to obtain the marginal probabilities for
multiple sets of query variables. For simplicity, we consider the sets of
single variables $Q_i \in \mathcal{Q}$, where $\mathcal{Q} = \{\{q_i\} \mid
q_i \in \Lambda \setminus E\}$. Using the above strategy would require
contracting $O(|\mathcal{Q}|)$ different tensor networks, which is
inefficient. In the following, we present an automatic
differentiation~\cite{liao2019differentiable} based approach to obtain the
marginal probabilities for all sets of variables in $\mathcal{Q}$ by
contracting the tensor network only once. The proposed algorithm reduces the
problem of finding marginal probability distributions to the problem of
finding the gradients of introduced auxiliary tensors, which can be
efficiently handled by differential programming. The differentiation rules for
tensor network contraction can be represented as the contraction of the tensor
network shown in \Cref{thm:diff}.
\begin{theorem}[Tensor network differentiation]\label{thm:diff}
    Let $(\Lambda, \mathcal{T}, \emptyset)$ be a tensor network with scalar
    output. The gradient of the tensor network contraction with respect to
    $T_V \in \mathcal{T}$ is
    \begin{equation}
      \frac{\partial \contract(\Lambda, \mathcal{T}, \emptyset)}{\partial T_V} =
      \contract(\Lambda, \mathcal{T} \setminus \{T_V\}, V).
    \end{equation}
    That is, the gradient corresponds to the contraction of the tensor network
    with the tensor $T_V$ removed and the output label set to $V$.
\end{theorem}

The proof of \Cref{thm:diff} is given in \Cref{sec:einback}.
The algorithm to obtain the marginal probabilities for all sets of variables
in $\mathcal{Q}$ is summarized as follows:
\begin{enumerate}
  \item Add a unity tensor $\mathbbm{1}_{Q_i}$ to the tensor network for each
    variable set $Q_i \in \mathcal{Q}$. A unity tensor is defined as a tensor
    with all elements equal to one. The augmented tensor network is
    represented as follows:
    \begin{equation}
        \mathcal{T}_{\text{aug}} \leftarrow \mathcal{T}
        \cup \{\mathbbm{1}_{Q_i} \mid Q_i \in \mathcal{Q}\}.
    \end{equation}
    The introduction of unity tensors does not change the contraction result
    of a tensor network.
  \item \textbf{Forward pass}: Contract the augmented tensor network to obtain
    \begin{equation}
      p(\ve) = \contract(\Lambda \setminus E, (\mathcal{T}_{\rm aug})_{\ve},
      \emptyset).
    \end{equation}
    In practice, the tensor network is contracted according to a given
    pairwise contraction order of tensors, caching intermediate results for
    later use. This order can be specified using a binary tree, which we will
    refer to as a binary contraction tree.
  \item \textbf{Backward pass:} Compute the gradients of the introduced unity
    tensors by back propagating the contraction process in Step~2. During
    back-propagation, the cached intermediate results from Step~2 are used.
    The resulting gradients are
    \begin{equation}
      \mathcal{G} = \left\{\frac{\partial{p(\ve)}}{\partial \mathbbm{1}_{Q_i}} \;\middle|\; Q_i \in \mathcal{Q}\right\}.
    \end{equation}
    Each gradient tensor $\partial{p(\ve)}/\partial{\mathbbm{1}_{Q_i}}$
    corresponds to a joint probability $p(Q_i, \ve)$. Dividing this gradient
    tensor by the partition function $p(\ve)$ yields the marginal probability
    $p(Q_i \mid \ve)$.
\end{enumerate}

In Step~1, we augment the tensor network by adding a rank 1 unity tensor for
each variable in $\mathcal{T}$. These tensors, being vectors, can be absorbed
into existing tensors of an optimized contraction tree, thereby not
considerably affecting the overall computing time. However, the computational
cost may increase significantly when unity tensors for joint marginal
probabilities of multiple variables are introduced.

The caching of intermediate contraction results in Step~2 is automatically
managed by a differential programming framework. These cached results are then
utilized in the back-propagation step. While this caching does not
significantly increase the computing time, it does lead to greater memory
usage. Practically, the added memory cost is typically just a few times
greater than the forward pass's peak memory. This is due to the program's
non-linear nature, often constrained by a handful of intensive contraction
steps. Step~3 follows from the observation that for any $Q_i \in \mathcal{Q}$,
the following holds:
\begin{equation}
  p(\ve) = \sum_{q \in \dom{Q_i}}p(\ve, Q_i=q)\slice{\mathbbm{1}}{L}{Q_i=q}.
\end{equation}
Using \Cref{thm:diff}, differentiating $p(\ve)$ with respect to
$\mathbbm{1}_{Q_i}$ is equivalent to removing the unity tensor
$\mathbbm{1}_{Q_i}$ from the tensor network and setting the output label to
$Q_i$, the result of which corresponds to the joint probability $p(Q_i, \ve)$.

\begin{corollary} \label{thm:complexity}

  Let $\mathcal{P}$ be a program to contract a tensor network $(\Lambda,
  \mathcal{T}, \emptyset)$ using a binary contraction tree. The time required
  to differentiate $\mathcal{P}$ using reverse-mode automatic differentiation
  is three times that required to evaluate $\mathcal{P}$.

\end{corollary}

\begin{proof}

  Since the program $\mathcal{P}$ is decomposed into a series of pairwise tensor
  contractions, to explain the overall factor of three, it suffices to show
  that for any pairwise tensor contraction, the computation time for
  backward-propagating gradients is twice that of the forward pass. Given a
  pairwise tensor contraction, $\contract(\Lambda, \{A_{V_a}, B_{V_b}\}, V_c)$,
  where $\Lambda = V_a \cup V_b \cup V_c$, its computational cost is
  $\prod_{v \in \Lambda} |\dom{v}|$, where $|\cdot|$ denotes the cardinality
  of a set. The backward rule for pairwise tensor contraction is also a tensor
  contraction. Let the adjoint of the output tensor be $\overline{C} \equiv
  \frac{\partial \mathcal{L}}{\partial C}$, where $\mathcal{L}$ is a scalar
  loss function, the explicit form of which does not need to be known. As
  shown in \Cref{sec:einback}, the backward rule for tensor contraction is
  \begin{equation} \label{eq:einback}
    \begin{split}
      \overline A_{V_a} = \contract(\Lambda, \{\overline{C}_{V_c},
      B_{V_b}\}, V_a)\\
      \overline B_{V_b} = \contract(\Lambda, \{A_{V_a},
      \overline{C}_{V_c}\}, V_b).
    \end{split}
  \end{equation}

  Since the above tensor networks share the same set of unique variables,
  their computing time is roughly equal to that of the forward computation.
  Consequently, the reverse-mode automatic differentiation for a tensor
  network is approximately three times more costly than computing only the
  forward pass, thus proving the theorem.

\end{proof}

\subsection{Most Probable Explanation (MPE)} \label{sec:mpe}

Consider the probabilistic model given by the tensor network $p(\Lambda) =
(\Lambda, \mathcal{T}, \Lambda)$, and suppose we are given evidence $\ve$,
where $E \subseteq \Lambda$. The objective of the Most Probable Explanation
(MPE) estimate is to determine the most likely assignment $q$ for the
variables $Q \in \Lambda \setminus E$. Mathematically, this can be expressed
as:
\begin{align*}
  \mathrm{MPE}(\ve) = \argmax_{q \in \dom{\Lambda \setminus E}} p(\vq, \ve),
\end{align*}
where the goal is not only to find the most likely assignment $\vq^*$ but also
to calculate its corresponding probability. In the subsequent discussion, we
will transition the tensor elements from real positive numbers to max-plus
numbers and reformulate the configuration extraction problem within the
context of differential programming.

\subsubsection{Tropical tensor networks}

Tropical algebra, a non-standard algebraic system, diverges from classical
algebra by replacing the standard operations of addition and multiplication
with different binary operations. In the following discussion, we focus on the
max-plus tropical algebra, a variant where the operations are \emph{maximum}
for addition and \emph{plus} for multiplication. 

\begin{definition}[Tropical Tensor Network]
  A \textit{tropical tensor network}~\cite{liu2021tropical} is a tensor
  network with max-plus tropical numbers as its tensor elements. Given two
  max-plus tropical numbers $a, b\in \mathbb{R} \cup \{-\infty\}$, their
  addition and multiplication operations are defined as
  \begin{align}
    \begin{split}  \label{eq:tropical}
      a \oplus b &= {\rm max}(a, b),\\
      a \odot b &= a + b.
    \end{split}
  \end{align}
  Correspondingly, the \emph{zero} element (or the additive identity) is
  mapped to $-\infty$, and the \emph{one} element (or the multiplicative
  identity) is mapped to $0$. Following from \Cref{eq:tropical} and
  \Cref{def:tnet}, the \textbf{tropical contraction} applied to a tropical
  tensor network $(\Lambda, \mathcal{T}, V_0)$ is defined as
  \begin{equation}
    \tropicalcontract(\Lambda, \mathcal{T}, V_0) =
    \max_{q \in \dom{\Lambda \setminus V_0}} \sum_{T_V \in \mathcal{T}} \slice{T}{V}{Q=q}.
  \end{equation}
  The $\max$ operation runs over all possible configurations over the set of
  variables absent in the output tensor.
\end{definition}

Given some evidence $\ve$, let us denote the MPE as $\vq^* \in \dom{\Lambda
\setminus E}$. The log probability of this MPE estimate can be computed as
follows:
\begin{equation}\label{eq:log-map}
  \begin{split}
    \log p(\vq^*, \ve) 
        &= \max_{q \in \dom{\Lambda\setminus E}}\log p(\vq, \ve),\\
        &= \max_{q \in \dom{\Lambda\setminus E}}\log \prod_{T_V\in\mathcal{T}}
        \slice{T}{V}{(Q=q, E=e)},\\ &= \max_{q \in \dom{\Lambda\setminus E}}\sum_{T_V\in\mathcal{T}}\log \slice{T}{V}{(Q=q, E=e)},\\
        &=\tropicalcontract(\Lambda \setminus E, \log(\mathcal{T})_{\ve}, \emptyset),
  \end{split}
\end{equation}
where $\log(\mathcal{T}) \equiv \{\log(T_V) \mid T_V \in \mathcal{T}\}$
represents the application of the logarithm operation to each tensor in the set
$\mathcal{T}$. The logarithm operation applied to a tensor is defined as
taking the logarithm of each element within the tensor.

\subsubsection{The most probable configuration}

In the context of the MPE estimate, the primary interest often lies not in
acquiring the log-probability of the MPE, calculated
according to \Cref{eq:log-map}, but rather in obtaining the configuration of
$\vq^*$ itself. The algorithm for this purpose is summarized as follows:
\begin{enumerate}
  \item For each variable $v \in \Lambda$, add a unity tensor $\mathbbm{1}_v$
    that is associated with $v$ to the tensor network. The augmented tensor
    network is given by:
    \begin{equation}
        \mathcal{T}_{\rm aug} \leftarrow \mathcal{T}
        \cup \{\mathbbm{1}_{\{v\}} \mid v \in \Lambda \setminus E\},
    \end{equation}
    Once more, we emphasize that introducing unity tensors does not change the
    contraction result of the tensor network.
  \item Evaluate the log-probability of the MPE, given by:
    \begin{multline}
      \log p(\vq^*, \ve) = \\
      \tropicalcontract(\Lambda \setminus E, \log(\mathcal{T}_{\rm aug})_{\ve}, \emptyset),
    \end{multline}
    where $\vq^* \in \dom{\Lambda\setminus E}$ is the MPE.
    Intermediate contraction results are cached for future use.
  \item Back-propagate through the contraction process outlined in Step~2 to
    obtain the gradients for each log-unity vector:
    \begin{equation}
      \begin{split}
        \mathcal{G} = \left\{\frac{\partial \log(p(\vq^*, \ve))}{\partial \log(\mathbb{I}_{\{v\}})} \;\middle|\; v \in \Lambda
      \right\}
      \end{split}
    \end{equation}
    Following a specific convention, we ensure that for each $G_{v} \in
    \mathcal{G}$, there exists exactly one non-zero entry, denoted as
    $G_v(q_v^*) = 1$. This unique entry $q_v^*$ corresponds to the assignment
    of variable $v$ in the MPE solution.
\end{enumerate}
Steps~1 and 2 of this approach mirror their counterparts in the algorithm
detailed in \Cref{sec:mar} for computing marginal probabilities, with the
notable distinction that tensor elements are now represented as tropical
numbers. Step~3 follows from the observation that, although the introduced
log-unity tensors (or zero tensors) do not affect the contraction result of
the tropical tensor network, differentiating the contraction result with
respect to these tensors yields a gradient signal at $\vq^*$.

\subsubsection{Back-propagation in tropical tensor networks}

Obtaining the MPE configuration using back-propagation through a tropical
tensor network is a non-trivial task, especially when there are multiple
configurations with the same maximum probability. In such cases, the gradient
signal must be designed to keep only one of the configurations. To achieve
this, we use a Boolean mask to represent the gradient signal of a tensor, with
its elements being either 0 or 1. In the following, instead of deriving the
exact back-propagation rule, we present a backward rule that only works for
Boolean gradients, which is sufficient for the MPE task.

\begin{theorem}
  Given a pairwise contraction of two tropical tensors
  $\tropicalcontract(\Lambda, \{A_{V_a}, B_{V_b}\}, V_c)$, where $\Lambda =
  V_a \cup V_b \cup V_c$, the backward rule, used for computing masks for nonzero gradients, is
  defined as follows:
\begin{equation} \label{eq:tropical-back}
  \begin{split}
    \overline A_{V_a} = \delta(A_{V_a}, \tropicalcontract(\Lambda, \{C_{V_c}^{-1}\overline{C}_{V_c}, B_{V_b}\}, V_a)^{-1})\\
    \overline B_{V_b} = \delta(B_{V_b}, \tropicalcontract(\Lambda, \{A_{V_a}, C_{V_c}^{-1}\overline{C}_{V_c}\}, V_b)^{-1})
  \end{split}
\end{equation}
\end{theorem}
\begin{proof}
  This rule is provable by reducing the tensor network contraction to tropical
  matrix multiplication $C = AB$. The back-propagation rule for tropical
  matrix multiplication has been previously derived
  in~\cite{liu2022computing}. Here, we revisit the main results for
  completeness. We require the gradients to be either $0$ or $1$. This binary
  nature aligns with the representation of configurations using onehot
  vectors. Consequently, a Boolean mask can effectively be employed to extract
  or represent any given configuration in this context. The gradient mask for
  $C$ is denoted as $\overline{C}$. The back-propagation rule for these
  gradient masks is:
  \begin{equation}\label{eq:adrule}
    \overline{A}_{ij} = \delta \left(A_{ij}, \left( \left( C^{\circ-1} \circ \overline{C} \right) B^{\mathsf{T}} \right)_{ij}^{\circ -1} \right),
  \end{equation}
  where $\delta$ is the Dirac delta function, returning one for equal
  arguments and zero otherwise. The notation $\circ$ signifies the
  element-wise product, and ${}^{\circ-1}$ indicates the element-wise inverse.
  Boolean \emph{false} is equated with the tropical zero ($-\infty$), and Boolean
  \emph{true} is the tropical one ($0$).
\end{proof}

In \Cref{eq:tropical-back}, the right-hand side primarily involves tropical
tensor contractions, which can be efficiently handled using fast tropical BLAS
routines~\cite{liu2023tropical}. Notably, the backward rule's computing time
mirrors that of the forward pass.

\subsection{Maximum Marginal a Posteriori (MMAP)} \label{sec:mmap}

Tensor networks are equally applicable in the context of computing
\emph{maximum marginal a posteriori} (MMAP) estimations. This task involves
computing the most likely assignment for the query variables, after
marginalizing out the remaining variables. Consider a scenario with evidence
$e$ observed over variables $E\subseteq \Lambda$ and query variables
$Q\subseteq \Lambda$, such that $E \cap Q = \emptyset$. Mathematically, the
MMAP solution is given by
\begin{multline} \label{eq:mmap}
  {\rm MMAP}(Q \mid \ve) = \\
    \argmax_{q \in \dom{Q}} \sum_{m \in \dom{\Lambda\setminus (Q, E)}} p(\vq, \vm, \ve).
\end{multline}
Upon closer examination of \Cref{eq:mmap}, it is clear that the equation
combines elements of both max-sum and sum-product networks. To optimize the
computation process, we utilize a routine that divides the computation into
two separate phases: conventional tensor network contraction and tropical
tensor network contraction. The process to compute MMAP solutions using tensor
networks is described as follows:
\begin{enumerate}
  \item Find a partition $\hat{\mathcal{S}}$ of $\mathcal{T}$ such that, for
    each marginalized variable $v \in \Lambda\setminus (Q\cup E)$, there
    exists an $\mathcal{S}_i \in \hat{\mathcal{S}}$ that contains all tensors
    associated with it.
  \item For each $\mathcal{S}_i \in \hat{\mathcal{S}}$, marginalize the
    variables in $\Lambda\setminus (Q \cup E)$ by contracting the tensor
    network
    \begin{equation}
      S_{\Lambda_i \cap Q} = \contract(\Lambda_i, (\mathcal{S}_i)_{\ve}, \Lambda_i\cap Q),
    \end{equation}
    where $\Lambda_i$ is the set of variables involved in $\mathcal{S}_i$.

  \item Solve the MPE problem on the probability model specified by tensor
    network $(Q, \{S_{\Lambda_1 \cap Q}, \ldots,
    S_{\Lambda_{|\hat{\mathcal{S}}|} \cap Q}\}, \emptyset)$, the result
    corresponds to the solution of the MMAP problem.

\end{enumerate}
In Step~1, The sets in $\hat{\mathcal{S}}$ can be constructed by first
choosing a marginalized variable $v \in \Lambda \setminus (Q\cup E)$, and then
greedily including tensors containing $v$ into the set.

\subsection{Generative modeling} \label{sec:sampling}

In the following discussion, we examine the use of tensor networks for
sampling from learned distributions in the context of probabilistic modeling.
This section aims to connect the contributions of this paper with other works
in the domain of tensor-based generative modeling, notably those by Han
et~al.~\cite{han2018unsupervised} and Cheng et~al.~\cite{cheng2019tree}. We
introduce a generic framework for unbiased variable sampling that generalizes
the sampling algorithms of these references. The algorithm for generating an
unbiased sample is summarized as follows:
\begin{enumerate}
  \item Contract the tensor network to obtain the partition function $p$. The
    contraction is done in a specified pairwise order, and intermediate
    results are cached for later use.
  \item Initialize a sample $\vs_{\emptyset}$ over an empty set of variables
    $\emptyset$.
  \item Trace back the contraction process and update the sample using the
    backward sampling rule in \Cref{eq:sample,eq:condition-prob}. Consider a
    pairwise tensor contraction $C_{Z} = \contract(X \cup Y, \{A_{X}, B_{Y}\}, Z)$,
    where $X, Y \subseteq \Lambda$ are the sets of variables involved in
    tensors $A$ and $B$, respectively, and $Z \subseteq X \cup Y$ are those
    involved in the output tensor $C$. The backward rule for generating
    samples entails obtaining an unbiased sample over the
    variables $X\cup Y$ given an unbiased sample over the variables in $Z$.
    This involves generating a sample $\vs_{X\cup Y} \sim p(X, Y \mid \vs_Z)$
    given another sample $\vs_Z \sim p(Z)$.
  \item Repeat Step~3 until all variables are sampled.
\end{enumerate}
In the following, we provide a detailed explanation of the backward sampling
rule in Step~3 of the algorithm. Let us first denote the set of variables
removed during the pairwise tensor contraction, $M = (X \cup Y) \setminus Z$, as
\emph{eliminated variables}. Let $\vs_Z \in \dom{Z}$ be an unbiased sample
over the variables in $Z$, i.e., $\vs_Z \sim p(Z)$. An unbiased sample
$\vs_{X \cup Y} \sim p(X, Y)$ can be obtained by first generating a sample
over the eliminated variables
\begin{equation}\label{eq:sample}
  \begin{split}
    \vs_{M} &\sim p(M \mid \vs_Z),
  \end{split}
\end{equation}
and then concatenating it with $\vs_Z$. The conditional probability can be
computed using Bayes' rule~\cite{joyce2021stanford} as
\begin{equation} \label{eq:condition-prob}
  \begin{split}
    p(M \mid \vs_Z) &= \frac{p(M, \vs_Z)}{p(\vs_Z)}\\
                    &= \frac{\contract(M, \{A_{X}, B_{Y}\}_{Z=\vs_Z}, {M})}{C_{\vs_Z}}.
  \end{split}
\end{equation}
In a valid tensor network contraction sequence, each variable is eliminated at
most once, preventing repeated sampling of any variable. In what follows, we
prove the correctness of \Cref{eq:condition-prob}.

\begin{proof}

  Let us first divide the tensors in $\mathcal{T}$ into three parts: those
  generating $A_{X}$, those generating $B_{Y}$, and the remaining part,
  $\mathcal{R}$. 
  By contracting the sub-tensor-networks associated with each of these three
  parts, three tensors $A_X$, $B_Y$, and $R_Z$ are
  obtained. Contracting these three tensors yields the partition function $p$
  as illustrated in \Cref{fig:threeparts}.

  \begin{figure}[th]
    \begin{tikzpicture}[
      mytensor/.style={
        circle,
        thick,
        draw=black!100,
        font=\small,
        minimum size=0.8cm
      },
      myedge/.style={
        line width=0.80pt,
      },
      ]
      \node (a) at (-1, 0) [mytensor] {$A_X$};
      \node (a1) at (-1.7, -0.7) {};
      \node (a2) at (-0.3, -0.7) {};
      \node at (-1, -0.6) {$\ldots$};
      \node (b) at (1, 0) [mytensor] {$B_Y$};
      \node (b1) at (0.3, -0.7) {};
      \node (b2) at (1.7, -0.7) {};
      \node at (1, -0.6) {$\ldots$};
      \node (c) at (0, 1) [mytensor] {$C_Z$};
      \node (r) at (2, 1) [mytensor] {$R_Z$};
      \node (r1) at (1.3, 0.3) {};
      \node (r2) at (2.7, 0.3) {};
      \node at (2, 0.4) {$\ldots$};
      \node (p) at (1, 2) [mytensor] {$p$};
      % parallel
      \draw [myedge] (a) edge node[below] {} (c);
      \draw [myedge] (b) edge node[below] {} (c);
      \draw [myedge] (c) edge node[below] {} (p);
      \draw [myedge] (r) edge node[below] {} (p);
      \draw [myedge] (a1) edge node[below] {} (a);
      \draw [myedge] (a2) edge node[below] {} (a);
      \draw [myedge] (b1) edge node[below] {} (b);
      \draw [myedge] (b2) edge node[below] {} (b);
      \draw [myedge] (r1) edge node[below] {} (r);
      \draw [myedge] (r2) edge node[below] {} (r);
    \end{tikzpicture}
    \caption{
      The contraction of a tensor network with three parts, $A_X$, $B_Y$, and
      $R_Z$.
    }
    \label{fig:threeparts}
  \end{figure}
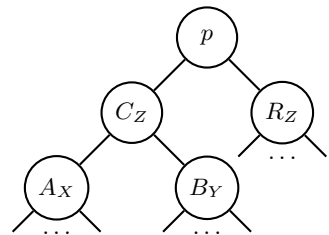

  %Let the set of variables eliminated during computing $A_{X}$ and
  %$B_{Y}$ as $\Lambda'$.
  %Here we emphasize that variables in $\Lambda' \cup ((X\cup Y)\setminus
  %Z)$ can not be involved in computing $R_{Z}$, otherwise, these
  %variables will be eliminated at least twice, contradicting the
  %assumption of valid contraction order.
  The marginal probabilities $p({X, Y})$ and $p(Z)$ can be obtained by by
  setting the output label to $X\cup Y$ and $Z$, respectively. This yields the
  following equations:
  \begin{equation} \label{eq:pz}
    \begin{split}
      p({X, Y}) &= \contract(X \cup Y, \{A_{X},
      B_{Y}, R_{Z}\}, {X \cup Y}),\\ p(Z) &= \contract(Z, \{C_{Z}, R_{Z}\}, Z).
    \end{split}
  \end{equation}
  By assigning $\vs_Z$ to $Z$, we obtain
  \begin{align}
    p(M, \vs_Z) &= \contract(M,  \{A_{X}, B_{Y}\}_{Z=\vs_Z}, M)R_{\vs_Z}, \label{eq:first-equation}\\
    p(\vs_Z) &= C_{\vs_Z}R_{\vs_Z}. \label{eq:second-equation}
  \end{align}
  By dividing \Cref{eq:first-equation} by \Cref{eq:second-equation}, we can
  readily obtain \Cref{eq:condition-prob}.

  % The generated sample $\vs_{A \cup B}$ can be further split into two samples
  % $\vs_A \sim p(X)$ and $\vs_B \sim p(Y)$. By propagating the sampling process
  % backwards into the contraction subtrees of $A$ and $B$, we obtain an unbiased
  % sample on $\Lambda$.

\end{proof}

This sampling algorithm is a natural generalization of those used in the
quantum-inspired probabilistic models, such as the matrix product state
ansatz~\cite{han2018unsupervised} and the tree tensor network
ansatz~\cite{cheng2019tree}. In quantum-inspired models, the involved tensors
are complex-valued and probabilities are represented using Born's rule. Born's
rule corresponds to contracting the complex tensor network with its conjugate,
as demonstrated in \Cref{fig:mpstree}. For example, in \Cref{subfig:mps}, the
contraction order of a matrix product state ansatz is from left to right, as
indicated by the dashed line. The variables are sampled in the reverse order
of elimination, i.e.,~from right to left. In a quantum-inspired ansatz, only
the ``physical'' variables (red edges) are sampled, while the ``virtual''
variables (black edges) are marginalized out. The tree tensor network ansatz,
shown in \Cref{sub@subfig:ttn}, is similar to the matrix product state ansatz
but features a different tensor network structure.

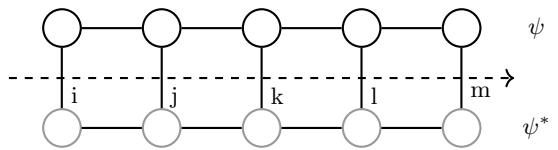
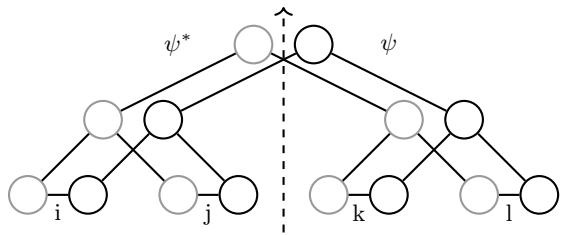
\begin{figure}
  \centering
  \begin{subfigure}[b]{\columnwidth}
    \begin{tikzpicture}[
      mytensor/.style={
        circle,
        thick,
        draw=black!100,
        font=\small,
        minimum size=0.5cm
      },
      myedge/.style={
        line width=0.80pt,
      },
      ]
      \matrix[row sep=0.8cm,column sep=0.8cm,ampersand replacement= \& ] {
        \node (a) [mytensor] {};                                 \&
        \node (b) [mytensor] {};                                 \&
        \node (c) [mytensor] {};                                 \&
        \node (d) [mytensor] {};                                 \&
        \node (e) [mytensor] {};                                 \&\\
        \node (1) [mytensor, draw=black!040] {};                                 \&
        \node (2) [mytensor, draw=black!040] {};                                 \&
        \node (3) [mytensor, draw=black!040] {};                                 \&
        \node (4) [mytensor, draw=black!040] {};                                 \&
        \node (5) [mytensor, draw=black!040] {};                                 \&
                                                                    \\
      };
      \node[right of = e] (t) {$\psi$};
      \node[right of = 5] (s) {$\psi^*$};
      % parallel
      \draw [myedge, draw=black!040] (1) edge node[below] {} (2);
      \draw [myedge, draw=black!040] (2) edge node[below] {} (3);
      \draw [myedge, draw=black!040] (3) edge node[below] {} (4);
      \draw [myedge, draw=black!040] (4) edge node[below] {} (5);
      \draw [myedge] (a) edge node[below] {} (b);
      \draw [myedge] (b) edge node[below] {} (c);
      \draw [myedge] (c) edge node[below] {} (d);
      \draw [myedge] (d) edge node[below] {} (e);
      
      \draw [myedge,draw=c02] (1) edge node[below right] {i} coordinate[midway](a1) (a);
      \draw [myedge,draw=c02] (2) edge node[below right] {j} (b);
      \draw [myedge,draw=c02] (3) edge node[below right] {k} (c);
      \draw [myedge,draw=c02] (4) edge node[below right] {l} (d);
      \draw [myedge,draw=c02] (5) edge node[below right] {m} coordinate[midway](e5) (e);
      
      \draw [->, myedge, dashed] ([xshift=-20pt]a1) -- ([xshift=20pt]e5);
    \end{tikzpicture}
    \caption{
      Probabilistic interpretation of a matrix product state~(MPS) tensor
      network.
    }
    \label{subfig:mps} 
  \end{subfigure}
  \par\bigskip % force a bit of vertical whitespace
  \begin{subfigure}[b]{\columnwidth}
    \begin{tikzpicture}[
      mytensor/.style={
        circle,
        thick,
        draw=black!100,
        font=\small,
        minimum size=0.5cm
      },
      myedge/.style={
        line width=0.80pt,
      },
      ]
      \node (a) [mytensor] at (0, 0) {};
      \node (b) [mytensor] at (-2, -1) {};
      \node (c) [mytensor] at (2, -1) {};
      \node (d) [mytensor] at (-3, -2) {};
      \node (e) [mytensor] at (-1, -2) {};
      \node (f) [mytensor] at (1, -2) {};
      \node (g) [mytensor] at (3, -2) {};
      \node (1) [mytensor, draw=black!040] at (-0.8, 0) {};
      \node (2) [mytensor, draw=black!040] at (-2.8, -1) {};
      \node (3) [mytensor, draw=black!040] at (1.2, -1) {};
      \node (4) [mytensor, draw=black!040] at (-3.8, -2) {};
      \node (5) [mytensor, draw=black!040] at (-1.8, -2) {};
      \node (6) [mytensor, draw=black!040] at (0.2, -2) {};
      \node (7) [mytensor, draw=black!040] at (2.2, -2) {};
      % parallel
      \draw [myedge, draw=black!040] (1) edge node[below] {} (2);
      \draw [myedge, draw=black!040] (1) edge node[below] {} (3);
      \draw [myedge, draw=black!040] (2) edge node[below] {} (4);
      \draw [myedge, draw=black!040] (2) edge node[below] {} (5);
      \draw [myedge, draw=black!040] (3) edge node[below] {} (6);
      \draw [myedge, draw=black!040] (3) edge node[below] {} (7);
      \draw [myedge] (a) edge node[below] {} (b);
      \draw [myedge] (a) edge node[below] {} (c);
      \draw [myedge] (b) edge node[below] {} (d);
      \draw [myedge] (b) edge node[below] {} (e);
      \draw [myedge] (c) edge node[below] {} (f);
      \draw [myedge] (c) edge node[below] {} (g);
      
      \draw [myedge,draw=c02] (4) edge node[below] {i} (d);
      \draw [myedge,draw=c02] (5) edge node[below] {j} (e);
      \draw [myedge,draw=c02] (6) edge node[below] {k} (f);
      \draw [myedge,draw=c02] (7) edge node[below] {l} (g);
      
      \draw [->, myedge, dashed] (-0.4, -2.5) -- (-0.4, 0.5);
      \node[right of = a] (t) {$\psi$};
      \node[left of = 1] (s) {$\psi^*$};
    \end{tikzpicture}
    \caption{
      Probabilistic interpretation of a tree tensor network~(TTN).
    }
    \label{subfig:ttn}
  \end{subfigure}
  \caption{
    Probabilistic interpretation of popular tensor networks. Dashed arrows
    denote the variable elimination order. Red edges correspond to the variables
    of interest. The set of gray tensors is the complex conjugate of the black
    tensors.
  }
  \label{fig:mpstree}
\end{figure}

\section{Performance benchmarks} \label{sec:benchmarks}

This section presents a series of performance benchmarks comparing the runtime
of our tensor-based probabilistic inference library, namely
\emph{TensorInference.jl}~\cite{roa2023tensor}, against that of other
established solvers for probabilistic inference. We have selected two
open-source libraries written in C++ for this purpose, namely the
\emph{Merlin}~\cite{marinescu2022merlin} and
\emph{libDAI}~\cite{mooij2010libdai} solvers. Their positive results in past
UAI inference competitions~\cite{uai2010summary,vibhav2014uai} make them
representative examples of standard practices in the field. Additionally, we
have included \emph{JunctionTrees.jl}~\cite{roa2022partial}, an open-source
library written in Julia and the predecessor of \emph{TensorInference.jl}. The
inference tasks supported by the libraries used in the benchmark are
summarized in~\Cref{table:libsupport}.

\begin{table}[ht]
  \centering
  \begin{tabular}{lcccc} 
    \toprule
     & PR & MAR & MPE & MMAP \\ [0.5ex] 
     \midrule
    \emph{TensorInference.jl} & \checkmark & \checkmark & \checkmark & \checkmark \\
    \emph{Merlin}             & \checkmark & \checkmark & \checkmark & \checkmark \\
    \emph{libDAI}             & \checkmark & \checkmark & \checkmark & $\times$   \\
    \emph{JunctionTrees.jl}   & $\times$   & \checkmark & $\times$   & $\times$   \\ [1ex]
    \bottomrule
  \end{tabular}
  \caption{
    The inference tasks supported by the libraries used in the benchmark. See
    \Cref{sec:tensor-network-for-probabilistic-modeling} for descriptions of
    these tasks.
  }
  \label{table:libsupport}
\end{table}

In these experiments, we used the UAI 2014 inference competition's benchmark
suite, which comprises problem sets from various domains, including computer
vision, signal processing, and medical diagnosis. These benchmark problems
serve as a standardized testbed for algorithms dealing with uncertainty in AI.
For the PR, MAR, and MPE tasks, we used the UAI 2014 \emph{MAR} problem sets,
as they are suitable for exact inference tasks. On the other hand, we used the
UAI 2014 \emph{MMAP} problem sets for the MMAP task, as these contain specific
sets of query variables required for such task. However, since the \emph{MMAP}
problem sets were designed for approximate algorithms, we were unable to solve
some of these problems using our exact inference methods. For the CPU
experiments, we conducted benchmarks for all four tasks. For the GPU
experiments, we focused only on benchmarking the MMAP task, since the problems
of the other tasks in the UAI 2014 benchmark suite are either too large for
exact inference or too small to benefit from GPU acceleration. The CPU
experiments were conducted on an AMD Ryzen Threadripper PRO 3995WX 64-Cores
Processor operating at 3.7GHz and equipped with 256GiB of RAM. The GPU
experiments were conducted on an NVIDIA Quadro RTX 8000 with 48~GiB of VRAM.

\begin{figure*}[th]
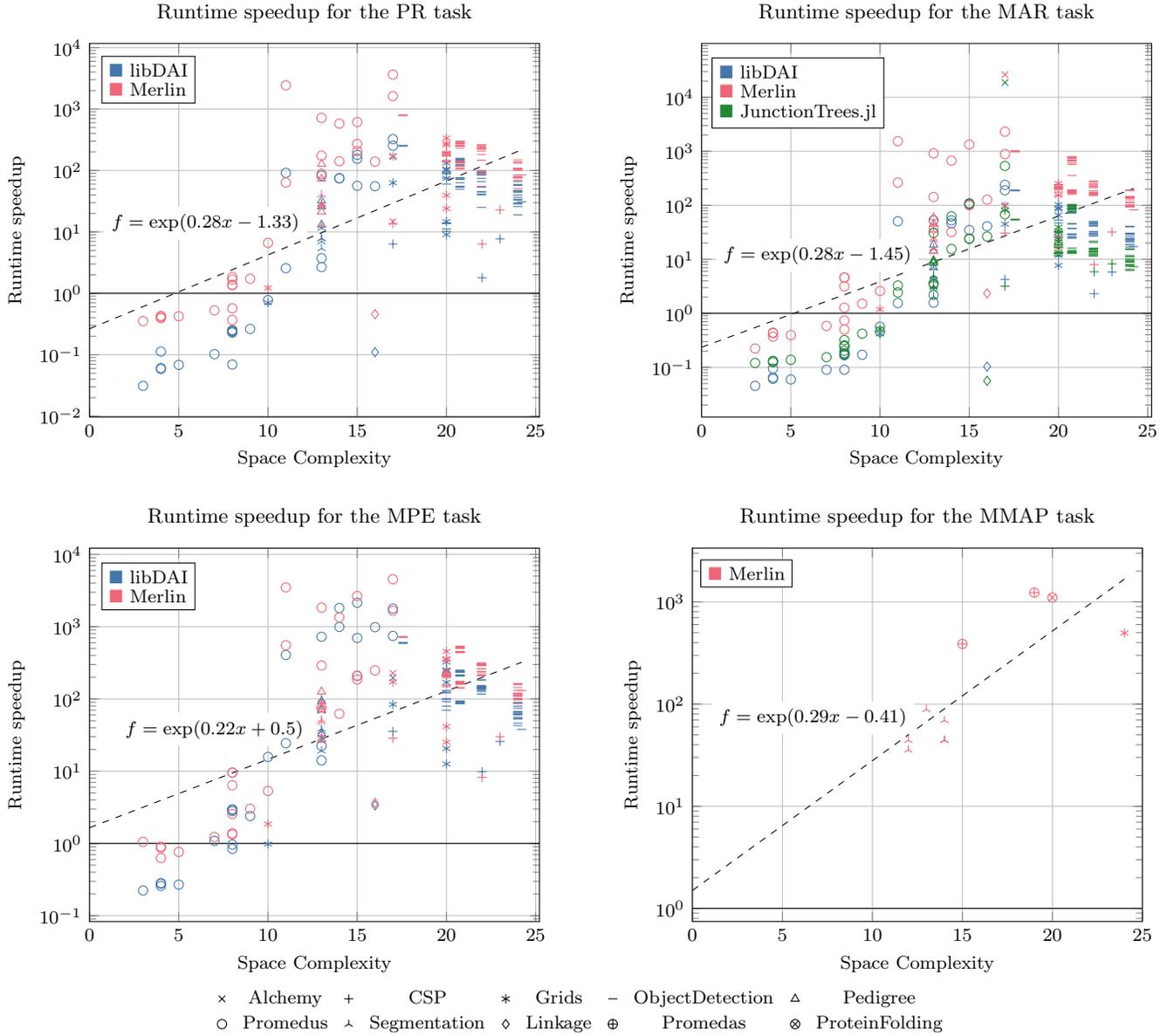

  %\captionsetup[subfigure]{font=small} if you like to change caption style
  \begin{subfigure}[b]{0.48\textwidth}
    \centering
    \input{scripts/benchmarks/pr_mar_mpe/out/ganzecheng-amat/2024-03-09--16-02-24/pr-ti_vs_merlin_vs_libdai.tex}
    \label{fig:pr-ti-vs-merlin-vs-libdai-vs-jt}
  \end{subfigure}
  \hfill
  \begin{subfigure}[b]{0.48\textwidth}
    \centering
    \input{scripts/benchmarks/pr_mar_mpe/out/ganzecheng-amat/2024-03-10--20-26-05/mar-ti_vs_merlin_vs_libdai.tex}
    \label{fig:mar-ti-vs-merlin-vs-libdai-vs-jt}
  \end{subfigure}

  \vskip\baselineskip
  \begin{subfigure}[b]{0.48\textwidth}
    \centering
    \input{scripts/benchmarks/pr_mar_mpe/out/ganzecheng-amat/2024-03-09--16-02-24/mpe-ti_vs_merlin_vs_libdai.tex}
    \label{fig:mpe-ti-vs-merlin-vs-libdai-vs-jt}
  \end{subfigure}
  \hfill
  \begin{subfigure}[b]{0.48\textwidth}
    \centering
    % Recommended preamble:
\begin{tikzpicture}
\begin{axis}[title={Runtime speedup for the MMAP task}, xmin={0}, xmax={25.0}, xlabel={Space Complexity}, xmajorgrids={true}, ymode={log}, ymajorgrids={true}, ylabel={Runtime speedup}, label style={font={\footnotesize}}, tick label style={font={\footnotesize}}, scatter/classes={Grids={mark={asterisk}}, Promedas={mark={oplus}}, ProteinFolding={mark={otimes}}, Segmentation={mark={Mercedes star}}}, legend style={legend columns={3}, at={(0.51,1.0)
}, anchor={south}, draw={none}, font={\footnotesize}, column sep={1.5}}]
    \addplot[c02, scatter, only marks, scatter src={explicit symbolic}, legend image post style={black}, legend style={text={black}, font={\footnotesize}}]
        table[row sep={\\}, meta={label}]
        {
            x  y  label  \\
            12.0  35.5646996632968  Segmentation
              \\
            12.0  44.174676232832255  Segmentation
              \\
            13.0  88.63474799024657  Segmentation
              \\
            14.0  44.074664302918976  Segmentation
              \\
            14.0  69.22491979986624  Segmentation
              \\
            14.0  44.39010117906956  Segmentation
              \\
            15.0  387.6847713904313  Promedas
              \\
            19.0  1231.4559806288432  Promedas
              \\
            20.0  1103.02014242042  ProteinFolding
              \\
            24.0  496.4986330768476  Grids
              \\
        }
        ;
    \addplot[black, no markers, dashed]
        table[row sep={\\}]
        {
            x  y  \\
            0.0  1.5011850614945426  \\
            0.24242424242424243  1.6114099470120709  \\
            0.48484848484848486  1.7297281220906187  \\
            0.7272727272727273  1.856733838523789  \\
            0.9696969696969697  1.9930649811905385  \\
            1.2121212121212122  2.1394062718252913  \\
            1.4545454545454546  2.2964927080256703  \\
            1.696969696969697  2.4651132547702255  \\
            1.9393939393939394  2.646114806986762  \\
            2.1818181818181817  2.8404064430732383  \\
            2.4242424242424243  3.0489639907344843  \\
            2.6666666666666665  3.2728349280666142  \\
            2.909090909090909  3.513143644504785  \\
            3.1515151515151514  3.7710970880573402  \\
            3.393939393939394  4.047990827189527  \\
            3.6363636363636362  4.345215557802523  \\
            3.878787878787879  4.664264087989024  \\
            4.121212121212121  5.006738835646229  \\
            4.363636363636363  5.374359876602926  \\
            4.606060606060606  5.768973583682311  \\
            4.848484848484849  6.192561900090122  \\
            5.090909090909091  6.647252293703609  \\
            5.333333333333333  7.135328442256644  \\
            5.575757575757576  7.6592417030871855  \\
            5.818181818181818  8.221623425053805  \\
            6.0606060606060606  8.825298164457728  \\
            6.303030303030303  9.47329787134726  \\
            6.545454545454546  10.168877117455082  \\
            6.787878787878788  10.91552944225066  \\
            7.03030303030303  11.717004899205616  \\
            7.2727272727272725  12.57732889039792  \\
            7.515151515151516  13.500822384051677  \\
            7.757575757575758  14.492123616554638  \\
            8.0  15.55621138795191  \\
            8.242424242424242  16.698430067916902  \\
            8.484848484848484  17.92451643779202  \\
            8.727272727272727  19.24062850351274  \\
            8.969696969696969  20.65337642412792  \\
            9.212121212121213  22.16985571125425  \\
            9.454545454545455  23.797682866208916  \\
            9.696969696969697  25.54503363380771  \\
            9.93939393939394  27.42068406495749  \\
            10.181818181818182  29.434054594279964  \\
            10.424242424242424  31.59525735414567  \\
            10.666666666666666  33.91514696275287  \\
            10.909090909090908  36.40537504133356  \\
            11.151515151515152  39.078448734298945  \\
            11.393939393939394  41.947793526241085  \\
            11.636363636363637  45.027820671288474  \\
            11.878787878787879  48.33399957347882  \\
            12.121212121212121  51.88293548167853  \\
            12.363636363636363  55.692452889271046  \\
            12.606060606060606  59.781685057487984  \\
            12.848484848484848  64.17117011201306  \\
            13.090909090909092  68.88295419550272  \\
            13.333333333333334  73.94070219410689  \\
            13.575757575757576  79.36981659411099  \\
            13.818181818181818  85.19756506565736  \\
            14.06060606060606  91.45321741433247  \\
            14.303030303030303  98.16819258845827  \\
            14.545454545454545  105.37621648043115  \\
            14.787878787878787  113.11349131466248  \\
            15.030303030303031  121.41887747287136  \\
            15.272727272727273  130.3340886699439  \\
            15.515151515151516  139.903901460629  \\
            15.757575757575758  150.17638012931536  \\
            16.0  161.2031180923954  \\
            16.242424242424242  173.03949702565833  \\
            16.484848484848484  185.74496501817546  \\
            16.727272727272727  199.38333514970506  \\
            16.96969696969697  214.0231059912157  \\
            17.21212121212121  229.73780563824087  \\
            17.454545454545453  246.6063610049672  \\
            17.696969696969695  264.71349423383424  \\
            17.939393939393938  284.15014821160605  \\
            18.181818181818183  305.01394332906574  \\
            18.424242424242426  327.4096677784048  \\
            18.666666666666668  351.4498038508208  \\
            18.90909090909091  377.2550928776429  \\
            19.151515151515152  404.9551416524046  \\
            19.393939393939394  434.6890733796041  \\
            19.636363636363637  466.6062264195406  \\
            19.87878787878788  500.866904338661  \\
            20.12121212121212  537.6431810325442  \\
            20.363636363636363  577.1197649652361  \\
            20.606060606060606  619.4949268655711  \\
            20.848484848484848  664.9814955398325  \\
            21.09090909090909  713.8079268021976  \\
            21.333333333333332  766.2194508916698  \\
            21.575757575757574  822.4793041383805  \\
            21.818181818181817  882.8700510652999  \\
            22.060606060606062  947.6950035655902  \\
            22.303030303030305  1017.2797442834021  \\
            22.545454545454547  1091.9737618493007  \\
            22.78787878787879  1172.1522061832413  \\
            23.03030303030303  1258.217772681111  \\
            23.272727272727273  1350.6027247481272  \\
            23.515151515151516  1449.7710648372647  \\
            23.757575757575758  1556.2208648967057  \\
            24.0  1670.486767930974  \\
        }
        ;
    \draw[solid, black, line width={0.4pt}] ({rel axis cs:1,0}|-{axis cs:0,1}) -- ({rel axis cs:0,0}|-{axis cs:0,1});
    \node 
    [draw={black}, fill={white}, font={\footnotesize}]  at 
    (3.0,1900)
    {\shortstack[l] { $\textcolor{c02}{\blacksquare}$ Merlin }};
    \node 
    [fill={white}, font={\footnotesize}]  at 
    (6.7,74)
    {$f = \exp(0.29x - 0.41)$};
\end{axis}
\end{tikzpicture}
    \label{fig:mmap-ti-vs-merlin}
  \end{subfigure}

  \begin{subfigure}{\textwidth}
    \centering
    \begin{tikzpicture}
  \begin{axis}[
    hide axis,
    height=46px,
    xmin=0,
    xmax=1,
    ymin=0,
    ymax=1,
    legend style={
      legend columns={5},
      at={(0.51,1.0)},
      anchor={south},
      draw={none},
      font={\footnotesize},
      column sep={5.0}
    }
    ]
    % Add legend entries
    \addlegendimage{mark={x}, color=black, only marks}
    \addlegendentry{Alchemy}
    \addlegendimage{mark={+}, color=black, only marks}
    \addlegendentry{CSP}
    \addlegendimage{mark={asterisk}, color=black, only marks}
    \addlegendentry{Grids}
    \addlegendimage{mark={-}, color=black, only marks}
    \addlegendentry{ObjectDetection}
    \addlegendimage{mark={triangle}, color=black, only marks}
    \addlegendentry{Pedigree}
    \addlegendimage{mark={o}, color=black, only marks}
    \addlegendentry{Promedus}
    \addlegendimage{mark={Mercedes star}, color=black, only marks}
    \addlegendentry{Segmentation}
    \addlegendimage{mark={diamond}, color=black, only marks}
    \addlegendentry{Linkage}
    \addlegendimage{mark={oplus}, color=black, only marks}
    \addlegendentry{Promedas}
    \addlegendimage{mark={otimes}, color=black, only marks}
    \addlegendentry{ProteinFolding}

  \end{axis}
\end{tikzpicture}
    \label{fig:performance-benchmark-legend}
  \end{subfigure}

  \caption{
    Runtime speedup achieved by our tensor-based library,
    \emph{TensorInference.jl}, across four different probabilistic inference
    tasks, relative to \emph{Merlin}~\cite{marinescu2022merlin},
    \emph{libDAI}~\cite{mooij2010libdai} and
    \emph{JunctionTrees.jl}~\cite{roa2022partial}. The experiments were
    conducted on a CPU using the UAI 2014 inference competition benchmark
    problems.
  }%

  \label{fig:performance-results}
\end{figure*}

The benchmark results, conducted on a CPU, are presented in
\Cref{fig:performance-results}, where each subfigure displays the results for
each of the considered inference tasks. The benchmark problems are arranged
along the x-axis in ascending order of the network's \emph{space complexity}.
This metric is defined as the logarithm base 2 of the number of elements in
the largest tensor encountered during contraction with a given optimized
contraction order. A common pattern observed among these four benchmark
results is that, as the complexity of the problem increases, our TN-based
implementation progressively outperforms the reference libraries. The
improvement is attributed to the tensor network contraction order algorithm,
which reduces the space complexity, time complexity and the read-write
complexity (the number of read and write operations) of the contraction at the
same time. The graphs feature a fitted linear curve in log space to underscore
this exponential improvement. However, for other less complex problems (those
with space complexities smaller than~$10$), our library generally performs
slower than the reference libraries. The reason is that the hyper-optimized
contraction order-finding algorithms in our library incur a cost that becomes
non-negligible for small-sized problems.

\Cref{fig:mmap-gpu-vs-cpu} demonstrates the speedups achieved by executing our
tensor-based method on a GPU versus on a CPU across different problem sizes
for the MMAP task. The results indicate that for large problem sizes, the
GPU-based implementation can improve \emph{TensorInference.jl}'s performance
by one to two orders of magnitude. However, for tasks with small problem
sizes, the overhead associated with transferring data between the CPU and GPU,
along with the time to launch GPU kernels, outweighs the advantages of using
the GPU, resulting in decreased performance. This finding aligns with the
observation that when space complexity is high, a few steps of tensor
contraction operations become the most time-consuming parts, and GPUs are
especially effective in accelerating these operations.

\begin{figure}
  \centering
  % Recommended preamble:
\begin{tikzpicture}
\begin{axis}[title={GPU vs. CPU Runtime Speedup}, xmin={10}, xmax={29.0}, xlabel={Space Complexity}, xmajorgrids={true}, ymax={100}, ymode={log}, ymajorgrids={true}, ylabel={Runtime speedup}, label style={font={\footnotesize}}, tick label style={font={\footnotesize}}, scatter/classes={Grids={mark={asterisk}}, Promedas={mark={oplus}}, ProteinFolding={mark={otimes}}, Segmentation={mark={Mercedes star}}}, legend style={legend columns={4}, at={(0.51,-0.26)
}, anchor={south}, draw={none}, font={\footnotesize}, column sep={1.5}}]
    \addplot[black, scatter, only marks, scatter src={explicit symbolic}, legend image post style={black}, legend style={text={black}, font={\footnotesize}}]
        table[row sep={\\}, meta={label}]
        {
            x  y  label  \\
            12.0  0.1771556743905404  Segmentation
              \\
            12.0  0.1761416504144137  Segmentation
              \\
            13.0  0.18536737498115172  Segmentation
              \\
            13.0  0.18775563146376795  Segmentation
              \\
            14.0  0.2028353808762056  Segmentation
              \\
            14.0  0.20549911648503189  Segmentation
              \\
            15.0  0.21249822783693526  Promedas
              \\
            19.0  0.9023801535430181  Promedas
              \\
            19.0  0.8230198687127651  Promedas
              \\
            19.0  1.1203373283743392  Promedas
              \\
            20.0  2.397816315466137  Promedas
              \\
            20.0  2.278621745828834  ProteinFolding
              \\
            21.0  3.4561401478293283  Promedas
              \\
            21.0  3.843389605019409  Promedas
              \\
            22.0  3.5770780062533505  Promedas
              \\
            22.0  5.179359860817827  Promedas
              \\
            23.0  6.1445301917104596  Promedas
              \\
            23.0  6.723916651718724  Promedas
              \\
            24.0  14.49230380278894  Grids
              \\
            24.0  11.365651536922615  Promedas
              \\
            24.0  7.354762983728186  Promedas
              \\
            24.0  9.277698799632528  Promedas
              \\
            24.0  9.041243412982242  Promedas
              \\
            25.0  11.064519044204596  Promedas
              \\
            25.0  8.973945463310118  Promedas
              \\
            27.0  12.74005929078088  Promedas
              \\
            27.0  23.303894653248367  Promedas
              \\
            27.0  11.90081459712289  Promedas
              \\
            27.0  10.944124811088697  Promedas
              \\
            28.0  21.06310890623099  Promedas
              \\
            28.0  12.447245155936887  Promedas
              \\
            28.0  15.627680484103992  Promedas
              \\
        }
        ;
    \draw[solid, black, line width={0.4pt}] ({rel axis cs:1,0}|-{axis cs:0,1}) -- ({rel axis cs:0,0}|-{axis cs:0,1});
    \legend{{Grids},{Promedas},{ProteinFolding},{Segmentation}}
\end{axis}
\end{tikzpicture}
  \caption{
    \emph{TensorInference.jl}'s runtime speedup on a GPU for the MMAP task,
    relative to CPU performance, benchmarked on the UAI 2014 inference
    competition problems.
  }
  \label{fig:mmap-gpu-vs-cpu}
\end{figure}
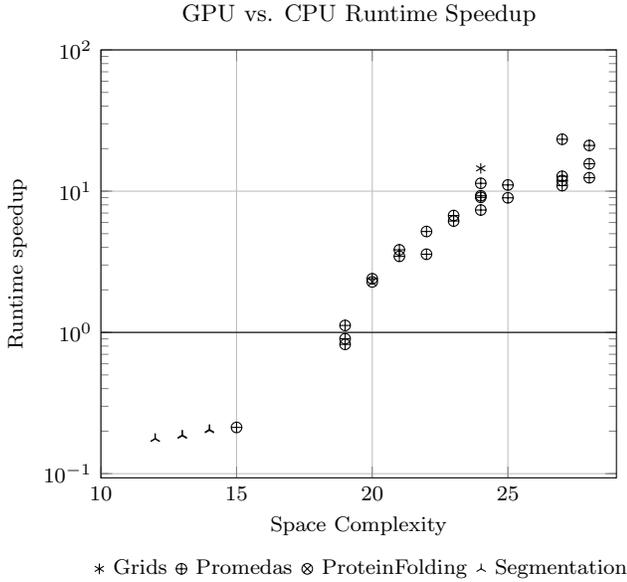

\section{Conclusions} \label{sec:conclusions}

We have formulated a series of prevalent probabilistic inference tasks in
terms of tensor network contractions and provided their corresponding
implementations. Our proposed formulation streamlines analogue formulations
encountered in classical Bayesian inference methods, including the family of
junction tree algorithms, by abstracting notions based on message passing and
by leveraging established tools such as differential programming frameworks.
We have shown how adjusting the algebraic system of a tensor network can be
used to solve different probabilistic inference tasks. We introduced the
unity-tensor approach to efficiently compute the marginal probabilities of
multiple variables using automatic differentiation. We also demonstrated how
tropical tensor network representations can be employed for computing the most
likely assignment of variables. Additionally, we unified the previously
developed sampling algorithms for chain and tree tensor networks.

As a product of this research, we have provided an implementation of our
proposed methods in the form of a Julia package, namely
\texttt{TensorInference.jl}~\cite{roa2023tensor}. Our library integrates the
latest developments in tensor network contraction order finding algorithms
from quantum computing into probabilistic inference. Moreover, our tensor
contraction implementation naturally compiles to BLAS functions, enabling us
to fully utilize the computational power of hardware such as CPUs, GPUs, and
TPUs, although the latter was not tested in this work.
%We implemented tropical GEMM on CUDA devices, which improves the performance
%of our framework in solving MPE and MMAP tasks by several orders of
%magnitude.

We conducted a comparative evaluation against three other open-source
libraries for probabilistic inference. Our method demonstrated substantial
speedups in runtime performance compared to the reference libraries across
various probabilistic tasks. Notably, the improvements became more pronounced
as the model complexity increased. These results underscore the potential of
our method in broadening the tractability spectrum of exact inference for
increasingly complex models.

As a future direction, we plan to utilize the tensor network framework to
speed up the quantum error correction process~\cite{ferris2014tensor}. The
quantum error correction process can be formulated as an MPE problem, where
the goal is to find the most likely error pattern given the observed syndrome
(or evidence).

\begin{acknowledgments}
  This work is partially funded by the Netherlands Organization for Scientific
  Research (P15-06 project 2) and the Guangzhou Municipal Science and
  Technology Project (No. 2023A03J0003 and No. 2024A04J4304). The authors
  thank Madelyn Cain and Pan Zhang for valuable advice, and Zhong-Yi Ni for
  insightful discussions on quantum error correction. We acknowledge the use
  of AI tools like Grammarly and ChatGPT for sentence rephrasing and grammar
  checks.
\end{acknowledgments}

\appendix

\section{Backward rule for tensor contraction} \label{sec:einback}

In this section, we will derive ${\text{\Cref{eq:einback}}}$, which is the
backward rule for a pairwise tensor contraction, denoted by
${\text{\contract}}(\Lambda, {A_{V_a}, B_{V_b}}, V_c)$.
Let $ \mathcal{L} $ be a loss function of interest, where its differential
form is given by:
\begin{align}\label{eq:diffeq}
    \delta\mathcal{L} &= \begin{multlined}[t]
      \contract(V_a, \{\delta A_{V_a},
      \overline{A}_{V_a}\}, \emptyset) + \contract(V_b,
      \{\delta B_{V_b}, \overline{B}_{V_b}\}, \emptyset)
  \end{multlined}\nonumber\\
                      &= \contract(V_c, \{\delta C_{V_c},
                      \overline{C}_{V_c}\}, \emptyset).
\end{align}
The goal is to find $\overline{A}_{V_a}$ and $\overline{B}_{V_b}$ given $\overline{C}_{V_c}$.
This can be achieved by using the differential form of tensor contraction, which states that
\begin{equation}
  \begin{split}
    \delta C = & \contract(\Lambda, \{\delta A_{V_a}, B_{V_b}\},
    V_c)\\
               &+ \contract(\Lambda, \{A_{V_a}, \delta B_{V_b}\},
               V_c).
  \end{split}
\end{equation}
By inserting this result into \Cref{eq:diffeq}, we obtain:
\begin{align}
  \delta\mathcal{L} &=
  \begin{multlined}[t]
    \contract(V_a, \{\delta A_{V_a}, \overline{A}_{V_a}\}, \emptyset)\\ + \contract(V_b, \{\delta B_{V_b}, \overline{B}_{V_b}\}, \emptyset)
  \end{multlined}\nonumber\\
                    &= \begin{multlined}[t] 
                      \contract(\Lambda, \{\delta A_{V_a},
                      B_{V_b}, \overline{C}_{V_c}\}, \emptyset) \\
                      + \contract(\Lambda, \{A_{V_a}, \delta
                      B_{V_b}, \overline{C}_{V_c}\}, \emptyset).
                    \end{multlined}
\end{align}
Since $\delta A_{V_a}$ and $\delta B_{V_b}$ are arbitrary, the above equation
immediately implies \Cref{eq:einback}.

%\nocite{*}

\bibliography{bibliography}% Produces the bibliography via BibTeX.

\end{document}